\documentclass[10pt,twocolumn,letterpaper,table]{article}

\usepackage{cvpr}
\usepackage{times}
\usepackage{epsfig}
\usepackage{graphicx}
\usepackage{amsmath}
\usepackage{amssymb}
\usepackage{amsthm}
\usepackage{float}
\usepackage{booktabs}
\usepackage{xcolor}
\usepackage{multirow}
\usepackage[font=small,labelfont=bf]{caption}

\usepackage[pagebackref=true,breaklinks=true,letterpaper=true,colorlinks,bookmarks=false]{hyperref}

\cvprfinalcopy 

\def\cvprPaperID{636} 

\ifcvprfinal\fi
\begin{document}

\title{Parsing Occluded People by Flexible Compositions}

\author{Xianjie Chen\\
University of California, Los Angeles\\
Los Angeles, CA 90095\\
{\tt\small cxj@ucla.edu}
\and
Alan Yuille\\
University of California, Los Angeles\\
Los Angeles, CA 90095\\
{\tt\small yuille@stat.ucla.edu}
}

\newcommand{\BL}{{\mathbf{l}}}
\newcommand{\BT}{{\mathbf{t}}}
\newcommand{\BI}{{\mathbf{I}}}
\newcommand{\BV}{{\mathbf{v}}}
\newtheorem{lemma}{Lemma}

\maketitle

\begin{abstract}
This paper presents an approach to parsing humans when there is significant occlusion. We model humans using a graphical model which has a tree structure building on recent work~\cite{yang2013Pami, chen_nips14} and exploit the connectivity prior that, even in presence of occlusion, the visible nodes form a connected subtree of the graphical model. We call each connected subtree a flexible composition of object parts. This involves a novel method for learning occlusion cues. During inference we need to search over a mixture of different flexible models. By exploiting part sharing, we show that this inference can be done extremely efficiently requiring only twice as many computations as searching for the entire object (\ie, not modeling occlusion). We evaluate our model on the standard benchmarked ``We Are Family" Stickmen dataset and obtain significant performance improvements over the best alternative algorithms.
\end{abstract}

\section{Introduction}

Parsing humans into parts is an important visual task with many applications such as activity recognition~\cite{wang2013approach,yao2012coupled}. A common approach is to formulate this task in terms of graphical models where the graph nodes and edges represent human parts and their spatial relationships respectively. This approach is becoming successful on benchmarked datasets~\cite{yang2013Pami,chen_nips14}. But in many real world situations many human parts are occluded. Standard methods are partially robust to occlusion by, for example, using a latent variable to indicate whether a part is present and paying a penalty if the part is not detected, but are not designed to deal with significant occlusion. One of these models~\cite{chen_nips14} will be used in this paper as a \textit{base model}, and we will compare to it.

\begin{figure}
\centering
\includegraphics[width=\linewidth,trim=10 10 10 10,clip=true]{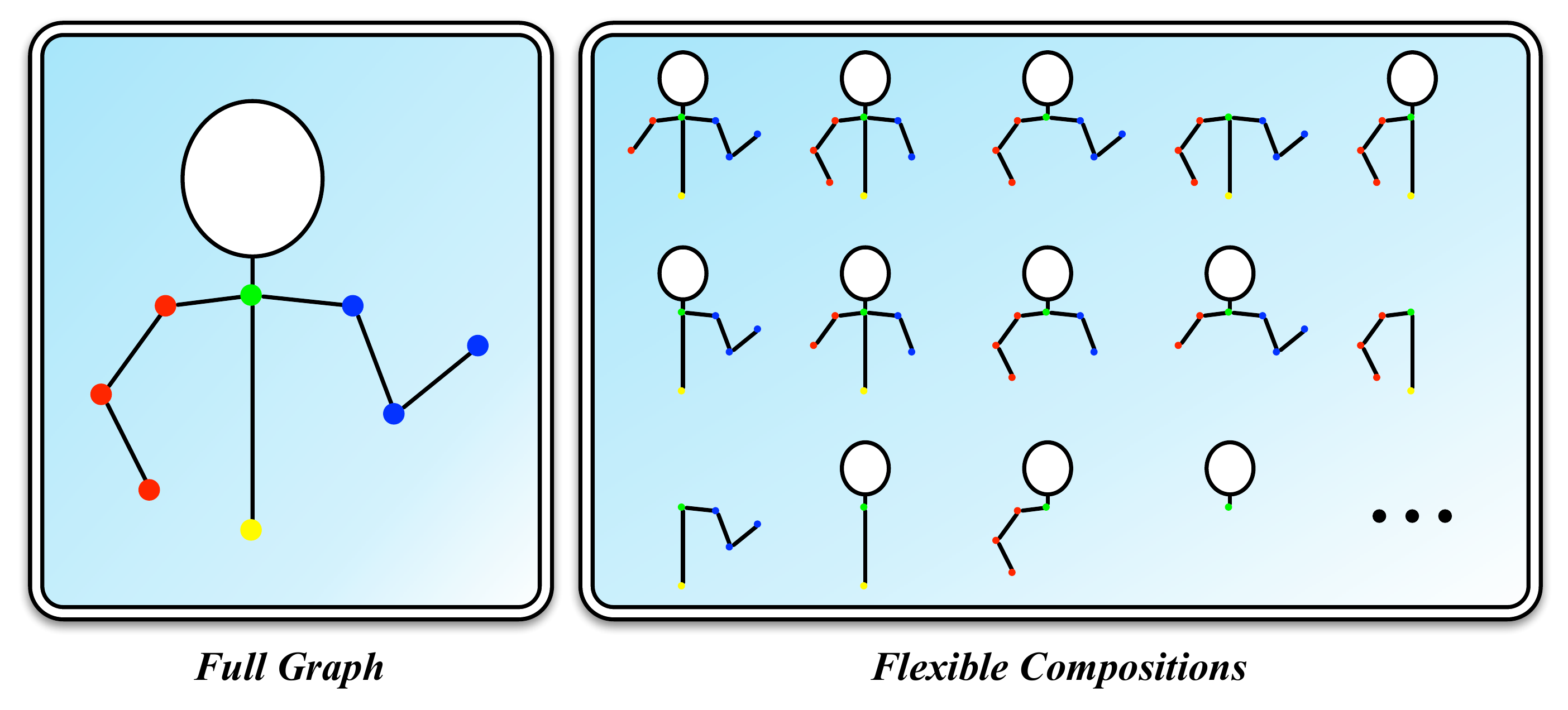}
\caption{An illustration of the \textit{flexible compositions}. Each connected subtree of the \textit{full graph} (include the full graph itself) is a flexible composition. The flexible compositions that do not have certain parts are suitable for the people with those parts occluded.}
\label{fig:model}
\end{figure}

In this paper, we observe that part occlusions often occur in regular patterns. The visible parts of a human tend to consist of a subset of connected parts even when there is significant occlusion (see Figures~\ref{fig:model} and ~\ref{fig:motivation}(a)). In the terminology of graphical models, the visible (non-occluded) nodes form a connected subtree of the full graphical model (following current models, for simplicity, we assume that the graphical model is treelike). This \emph{connectivity prior} is not always valid (\ie, the visible parts of humans may form two or more connected subsets), but our analysis (see Section~\ref{sec:diag}) suggests it's often true. In any case, we will restrict ourselves to it in this paper, since verifying that some isolated pieces of body parts belong to the same person is still very difficult for vision methods, especially in challenging scenes where multiple people occlude one another (see Figure~\ref{fig:motivation}).

To formulate our approach we build on the base model~\cite{chen_nips14}, which is the state of the art on several benchmarked datasets~\cite{Johnson10, sapp2013modec,ferrari2008progressive}, but is not designed for dealing with significant occlusion. We explicitly model occlusions using the connectivity prior above. This means that we have a mixture of models where the number of components equals the number of \emph{all} the possible connected subtrees of the graph, which we call \textit{flexible compositions}, see Figure~\ref{fig:model}. The number of flexible compositions can be large (for a simple chain like model consisting of $N$ parts, there are $N(N+1)/2$ possible compositions). Our approach exploits the fact there is often local evidence for the presence of occlusions, see Figure~\ref{fig:motivation}(b). We propose a novel approach which learns occlusion cues, which can break the links/edges, between adjacent parts in the graphical model. It is well known, of course, that there are local cues such as T-junctions which can indicate local occlusions. But although these occlusion cues have been used by some models (\eg, \cite{coughlan2000efficient,tu2006parsing}), they are not standard in graphical models of objects.

\begin{figure*}[ht]
\centering
\includegraphics[width=\linewidth,trim=10 10 10 10,clip=true]{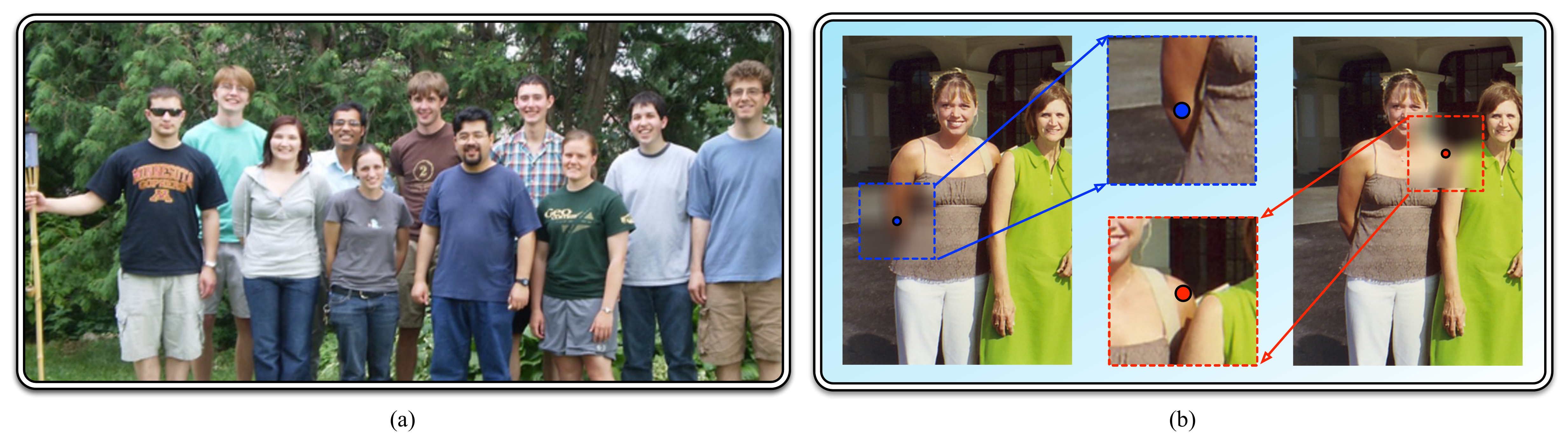}
\caption{Motivation. (a): In real world scenes, people are usually significantly occluded (or truncated). Requiring the model to localize a fixed set of body parts while ignoring the fact that different people have different degrees of occlusion (or truncation) is problematic. (b): The absence of body parts evidence can help to predict occlusion, \eg, the right wrist of the lady in brown can be inferred as occluded because of the absence of suitable wrist near the elbow. However, absence of evidence is not evidence of absence. It can fail in some challenging scenes, for example, even though the left arm of the lady in brown is completely occluded, there is still strong image evidence of suitable elbow and wrist at the plausible locations due to the confusion caused by nearby people (\eg, the lady in green). In both situations, the local image measurements near the occlusion boundary (\ie, around the right elbow and left shoulder), \eg, in a image patch, can reliably provide evidence of occlusion.}
\label{fig:motivation}
\end{figure*}

We show that efficient inference can be done for our model by exploiting the sharing of computation between different flexible models. Indeed, the complexity is only doubled compared to recent models where occlusion is not explicitly modeled. This rapid inference also enables us to train the model efficiently from labeled data.

We illustrate our algorithm on the standard benchmarked ``We Are Family" Stickmen (WAF) dataset~\cite{eichner2010we} for parsing humans when significant occlusion is present. We show strong performance with significant improvement over the best existing method~\cite{eichner2010we} and also outperform our base model~\cite{chen_nips14}. We perform diagnostic experiments to verify our connectivity prior that the visible parts of a human tend to consist of a subset of connected parts even when there is significant occlusion, and quantify the effect of different aspects of our model.


\section{Related work}
\label{sec:related}
Graphical models of objects have a long history~\cite{fischler1973representation,felzenszwalb2005pictorial}. Our work is most closely related to the recent work of Yang and Ramanan~\cite{yang2013Pami}, Chen and Yuille~\cite{chen_nips14}, which we use as our base model and will compare to. Other relevant work includes~\cite{pishchulin2013poselet, sapp2010adaptive, ferrari2008progressive, tompson2014joint}.

Occlusion modeling also has a long history~\cite{hsiao2012occlusion, Dollar2012PAMI}. Psychophysical studies (\eg, Kanizsa~\cite{kanizsa1979organization}) show that T-junctions are a useful cue for occlusion. But there has been little attempt to model the spatial patterns of occlusions for parsing objects. Instead it is more common to design models so that they are robust in the presence of occlusion, so that the model is not penalized very much if an object part is missing. Girshick et. al.~\cite{girshick2011object} and Supervised-DPM~\cite{azizpour2012object} model the occluded part (background) using extra templates. And they rely on a root part (\ie, the holistic object) that never takes the status of ``occluded". When there is significant occlusion, modeling the root part itself is difficult. Ghiasi et. al.~\cite{Ghiasi14_people} advocates modeling the occlusion area (background) using more templates (mixture of templates), and localizes every body parts. It may be plausible to ``guess" the occluded keypoints of face (\eg, ~\cite{burgos2013robust, Ghiasi14_face}), but seems impossible for body parts of people, due to highly flexible human poses. Eichner and Ferrari~\cite{eichner2010we} handles occlusion by modeling interactions between people, which assumes the occlusion is due to other people. 

Our approach models object occlusion effectively uses a mixture of models to deal with different occlusion patterns. There is considerable work which models objects using mixtures to deal with different configurations, see Poselets~\cite{BourdevMalikICCV09} which uses many mixtures to deal with different object configurations, and deformable part models (DPMs)~\cite{felzenszwalb2010object} where mixtures are used to deal with different viewpoints.

To ensure efficient inference, we exploit the fact that parts are shared between different flexible compositions.  This sharing of parts has been used in other work, e.g., ~\cite{chen_cvpr14}. Other work that exploits part sharing includes compositional models~\cite{zhu2010part} and AND-OR graphs~\cite{zhu2008max, zhu2006stochastic}.

\section{The Model}
\label{sec:model}
We represent human pose by a graphical model $\mathcal{G} = (\mathcal{V}, \mathcal{E})$ where the nodes $\mathcal{V}$ corresponds to the parts (or joints) and the edges $\mathcal{E}$ indicate which parts are directly related. For simplicity, we impose that the graph structure forms a $K-$node tree, where $K = |\mathcal{V}|$. The pixel location of part part $i$ is denoted by $\BL_i = (x, y)$, for $i \in \left\{1, \dots, K\right\}$. 

To model the spatial relationship between neighboring parts $(i,j) \in \mathcal{E}$, we follow the base model~\cite{chen_nips14} to discretize the pairwise spatial relationships into a set indexed by $t_{ij}$, which corresponds to a mixture of different spatial relationships. 

In order to handle people with different degrees of occlusion, we specify a binary occlusion decoupling variable $\gamma_{ij} \in \left\{0, 1\right\}$ for each edge $(i,j) \in \mathcal{E}$, which enables the subtree $\mathcal{T}_j = (\mathcal{V}(\mathcal{T}_j ), \mathcal{E}(\mathcal{T}_j))$ rooted at part $j$ to be decoupled from the graph at part $i$ (the subtree does not contain part $i$, \ie, $i \notin \mathcal{V}(\mathcal{T}_j )$). This results in a set of flexible compositions of the graph, indexed by set $\mathcal{C}_\mathcal{G}$. These compositions share the nodes and edges with the full graph $\mathcal{G}$ and each of themselves forms tree graph (see Figure~\ref{fig:model}). The compositions that do not have certain parts are suitable for the people with those parts occluded. 

In this paper, we exploit the connectivity prior that body parts tend to be connected even in the presence of occlusion, and do not consider the cases when people are separated into isolated pieces, which is very difficult. Handling these cases typically requires non-tree models, \eg,\cite{chen_cvpr14}, and thus does not have exact and efficient inference algorithms. Moreover, verifying whether some isolated pieces of people belong to the same person is still very difficult for vision methods, especially in challenging scenes where multiple people usually occlude one another (see Figure~\ref{fig:motivation}(a)).

For each flexible composition $\mathcal{G}_c = (\mathcal{V}_c, \mathcal{E}_c), c \in \mathcal{C}_\mathcal{G}$, we will define a score function $F(\BL, \BT, \mathcal{G}_c | \BI, \mathcal{G})$ as a sum of appearance terms, pairwise relational terms, occlusion decoupling terms and decoupling bias terms. Here $\BI$ denotes the image, $\BL=\left\{\BL_i | i\in \mathcal{V} \right\}$ is the set of locations of the parts, and $\BT = \left\{\BT_{ij}, \BT_{ji} | (i,j) \in \mathcal{E}  \right\} $ is the set of spatial relationships.

\noindent\textbf{Appearance Terms:}
The appearance terms make use of the local image measurement within patch $\BI(\BL_i)$ to provide evidence for part $i$ to lie at location $\BL_i$. They are of form:
\begin{equation}
A(\BL_i | \BI) =  w_i \phi(i | \BI(\BL_i); \boldsymbol\theta), 
\end{equation}
\noindent where $\phi(. | .; \boldsymbol\theta)$ is the (scalar-valued) appearance term with $\boldsymbol\theta$ as its parameters (specified in Section~\ref{sec:dcnn_id}), and $w_i$ is a scalar weight parameter.

\noindent\textbf{Image Dependent Pairwise Relational (IDPR) Terms:}
We follow the base model~\cite{chen_nips14} to use image dependent pairwise relational (IDPR) terms, which gives stronger spatial constraints between neighboring parts $(i,j) \in \mathcal{E}$. Stronger spatial constraints reduce the confusion from the nearby people and clustered background, which helps to better infer occlusion.

More formally, the relative positions between parts $i$ and $j$ are discretized into several types $t_{ij} \in \left\{ 1, \dots, T_{ij} \right\}$ (\ie, a mixture of different relationships) with corresponding mean relative positions $\mathbf{r}_{ij}^{t_{ij}}$ plus small deformations which are modeled by the standard quadratic deformation term. They are given by:
\begin{equation}
\begin{aligned}
R(\BL_i, \BL_j, t_{ij}, t_{ji} | \BI) &= \langle \mathbf{w}_{ij}^{t_{ij}}, \boldsymbol \psi(\BL_j-\BL_i - \mathbf{r}_{ij}^{t_{ij}}) \rangle 
\\
&+ w_{ij}  \varphi^s (t_{ij},\gamma_{ij}=0 | \BI(\BL_i); \boldsymbol\theta) 
\\
&+\langle \mathbf{w}_{ji}^{t_{ji}}, \boldsymbol\psi(\BL_i - \BL_j - \mathbf{r}_{ji}^{t_{ji}}) \rangle 
\\
&+ w_{ji} \varphi^s (t_{ji},\gamma_{ji}=0 | \BI(\BL_j); \boldsymbol\theta)
\end{aligned},
\label{eq:pairwise}
\end{equation}
\noindent where $\boldsymbol\psi(\Delta \BL=[ \Delta x, \Delta y]) = [\Delta x \; \Delta x^2 \; \Delta y \; \Delta y^2]^{\intercal}$ are the standard quadratic deformation features, $\varphi^s(.,\gamma_{ij}=0 | .; \boldsymbol\theta)$ is the Image Dependent Pairwise Relational (IDPR) term with $\boldsymbol\theta$ as its parameters (specified in Section~\ref{sec:dcnn_id}). IDPR terms are only defined when both part $i$ and $j$ are visible (\ie, $\gamma_{ij} = 0$ and $\gamma_{ji} = 0$). Here $\mathbf{w}_{ij}^{t_{ij}},w_{ij}, \mathbf{w}_{ji}^{t_{ji}},w_{ji}$ are the weight parameters, and the notation $\langle .,. \rangle$ specifies dot product and boldface indicates a vector. 

\noindent\textbf{Image Dependent Occlusion Decoupling (IDOD) Terms:}
These IDOD terms capture our intuition that the visible part $i$ near the occlusion boundary (and thus is a leaf node in each flexible composition) can reliably provide occlusion evidence using \emph{only} local image measurement (see Figure~\ref{fig:motivation}(b) and Figure~\ref{fig:idpr_idod}). More formally, the occlusion decoupling score for decoupling the subtree $\mathcal{T}_j$ from the full graph at part $i$ is given by: 
\begin{equation}
D_{ij}(\gamma_{ij} = 1, \BL_i | \BI) = w_{ij} \varphi^d(\gamma_{ij}=1 | \BI(\BL_i); \boldsymbol\theta),
\label{eq:idod}
\end{equation}
\noindent where $\varphi^d (\gamma_{ij}=1|. ; \boldsymbol\theta)$ is the Image Dependent Occlusion Decoupling (IDOD) term with $\boldsymbol\theta$ as its parameters (specified in Section~\ref{sec:dcnn_id}), $\gamma_{ij} = 1$ indicates subtree $\mathcal{T}_j$ is decoupled from the full graph. Here $w_{ij}$ is the scalar weight parameter shared with the IDPR term.

\noindent\textbf{Decoupling Bias Term:}
The decoupling bias term captures our intuition that the absence of evidence for suitable body part can help to predict occlusion. We specify a scalar bias term $b_i$ for each part $i$ as a learned measure for the absence of good part appearance, and also the absence of suitable spatial coupling with neighboring parts (our spatial constraints are also image dependent).

The decoupling bias term for decoupling the subtree $\mathcal{T}_j = (\mathcal{V}(\mathcal{T}_j ), \mathcal{E}(\mathcal{T}_j))$ from the full graph at part $i$, is defined as the sum of all the bias terms associated with the parts in the subtree, \ie, $k \in \mathcal{V}(\mathcal{T}_j )$. They are of form:
\begin{equation}
B_{ij} = \sum_{k \in \mathcal{V}(\mathcal{T}_j )}b_k
\end{equation}

\noindent\textbf{The Model Score:}
The model score for a person is the maximum score of all the flexible compositions $c \in \mathcal{C}_\mathcal{G}$, therefore the index $c$ of the flexible composition is also a random variable that need to be estimated, which is different from the standard graphical models with single graph structure.

The score $F(\BL, \BT, \mathcal{G}_c | \BI, \mathcal{G})$ for each flexible composition $c \in \mathcal{C}_\mathcal{G}$ is a function of the locations $\BL$, the pairwise spatial relation types $\BT$, the index of the flexible composition $c$, the structure of the full graph $\mathcal{G}$, and the input image $\BI$. It is given by:
\begin{equation}
\begin{aligned}
F(\BL, \BT, \mathcal{G}_c | \BI, \mathcal{G}) &= \sum_{i \in \mathcal{V}_c} A(\BL_i | \BI) 
\\
&+ \sum_{(i,j) \in \mathcal{E}_c} R(\BL_i, \BL_j, t_{ij}, t_{ji} | \BI) 
\\
&+ \sum_{(i,j)\in\mathcal{E}_c^d} ( B_{ij} + D_{ij}(\gamma_{ij} = 1, \BL_i  | \BI) ) 
\end{aligned}
\label{eq:full}
\end{equation}
\noindent where $\mathcal{E}_c^d = \left\{ (i, j)\in\mathcal{E} | i \in \mathcal{V}_c, j \notin \mathcal{V}_c \right\}$ is the edges that is decoupled to generate the composition $\mathcal{G}_c$. See Section~\ref{sec:learn} for the learning of the model parameters.

\subsection{Deep Convolutional Neural Network (DCNN) for Image Dependent Terms}
\label{sec:dcnn_id}
\begin{figure}
\centering
\includegraphics[width=\linewidth,trim=10 10 10 10,clip=true]{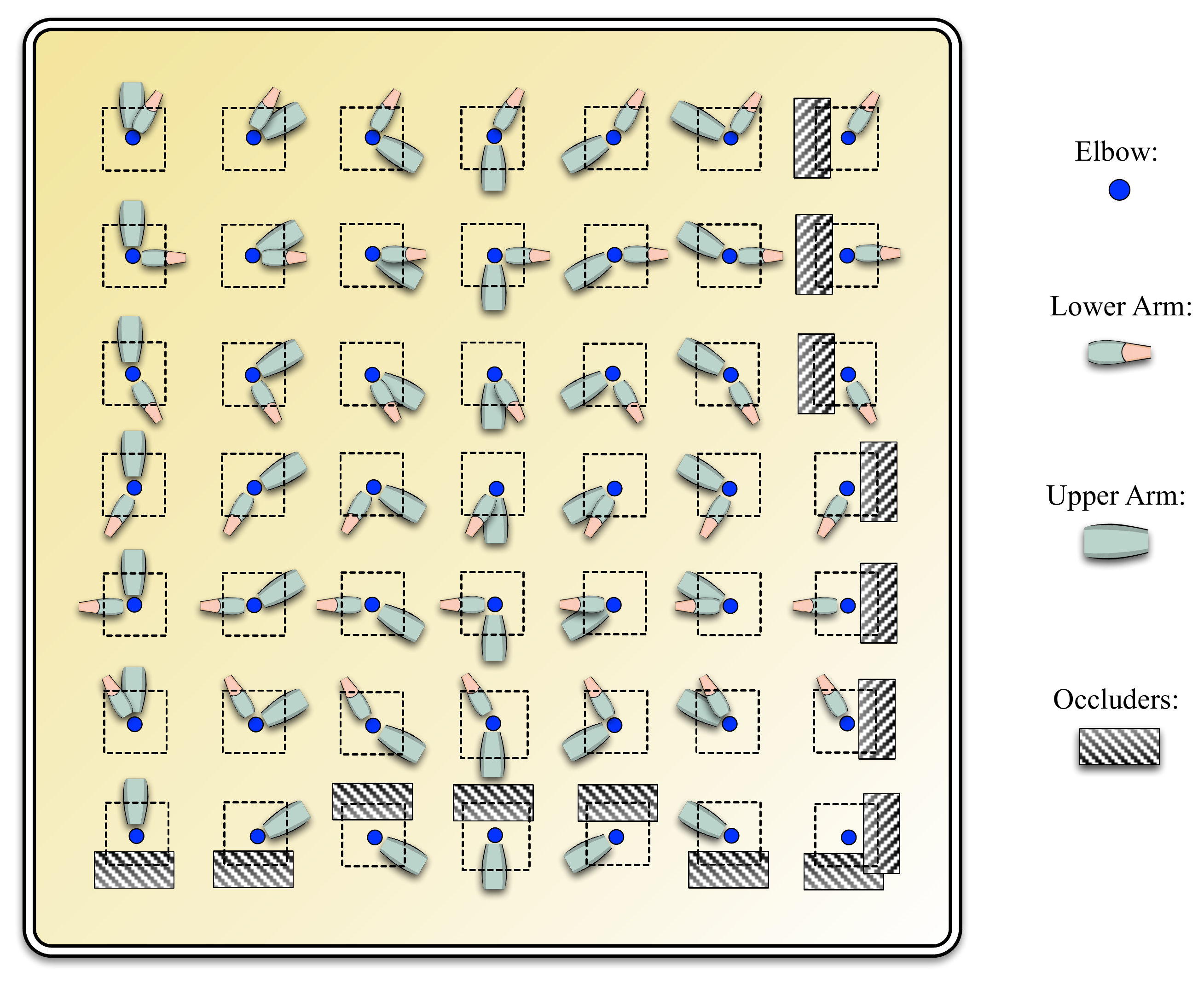}
\caption{Different occlusion decoupling and spatial relationships between the elbow and its neighbors, \ie, wrist and shoulder. The local image measurement around a part (\eg, the elbow) can reliably predict the relative positions of its neighbors when they are not occluded, which is demonstrated in the base model~\cite{chen_nips14}. In the case when the neighboring parts are occluded, the local image measurement can also reliably provide evidence for the occlusion.}
\label{fig:idpr_idod}
\end{figure}

Our model has three kinds of terms that depend on the local image patches: the appearance terms, IDPR terms and IDOD terms. This requires us to have a method that can efficiently extract information from a local image patch $\BI(\BL_i)$ for the presence of the part $i$, as well as the occlusion decoupling evidence $\gamma_{ij}=1$ of its neighbors $j \in \mathcal{N}(i)$, where $j \in \mathcal{N}(i)$ if, and only if, $(i,j) \in \mathcal{E}$. When a neighboring part $j$ is not occluded, \ie $\gamma_{ij}=0$, we also need to extract information for the pairwise spatial relationship type $t_{ij}$ between parts $i$ and $j$.

Extending the base model~\cite{chen_nips14}, we learn the distribution for the state variables $i, t_{ij}, \gamma_{ij}$ conditioned on the image patches $\BI(\BL_i)$. We'll first define the state space of this distribution. 

Let $g$ be the random variable that denotes which part is present, \ie, $g=i$ for part $i \in \left\{1,...,K\right\}$ or $g=0$ if no part is present (\ie, the background). We define $\mathbf{m}_{g \mathcal{N}(g)}=\left\{ m_{gk} | k \in \mathcal{N}(g)\right\}$ to be the random variable that determines the pairwise occlusion decoupling and spatial relationships between part $g$ and all its neighbors $\mathcal{N}(g)$, and takes values in $\mathcal{M}_{g \mathcal{N}(g)}$. If part $g=i$ has one neighbor $j$ (\eg, the wrist), then $\mathcal{M}_{i \mathcal{N}(i)} = \left\{0, 1, \dots, T_{ij}\right\}$, where the value $0$ represents part $j$ is occluded, \ie, $\gamma_{ij} = 1$ and the other values $v \in \mathcal{M}_{i \mathcal{N}(i)}$ represent part $j$ is not occluded and has corresponding spatial relationship types with part $i$, \ie, $\gamma_{ij} = 0, t_{ij}=v$. If $g=i$ has two neighbors $j$ and $k$ (\eg, the elbow), then $\mathcal{M}_{i \mathcal{N}(i)} = \left\{0, 1, \dots, T_{ij}\right\} \times \left\{0, 1, \dots, T_{ik}\right\}$ (Figure~\ref{fig:idpr_idod} illustrates the space $\mathcal{M}_{i \mathcal{N}(i)}$ for the elbow when $T_{ik}=T_{ij}=6$). If $g=0$, then we define $\mathcal{M}_{0 \mathcal{N}(0)} = \left\{0\right\}$. 

The full space can be written as:
\begin{equation}
\mathcal{U} = \cup_{g=0}^K \left\{ g \right\}\times \mathcal{M}_{g \mathcal{N}(g)}
\end{equation}
The size of the space is $|\mathcal{U}| = \sum_{g=0}^K |\mathcal{M}_{g \mathcal{N}(g)} |$. Each element in this space corresponds to the background or a part with a kind of occlusion decoupling configurations of all its neighbors and the types of its pairwise spatial relationships with its visible neighbors.

With the space of the distribution defined, we use a single Deep Convolutional Neural Network (DCNN)~\cite{krizhevsky2012imagenet}, which is efficient and effective for many vision tasks~\cite{zhang2014panda, girshick2014rich, Jay_George15_ICLR}, to learn the conditional probability distribution $p(g, \mathbf{m}_{g \mathcal{N}(g)} | \BI(\BL_i); \boldsymbol\theta)$. See Section~\ref{sec:learn} for more details.

We specify the appearance terms $\phi(. | .;\boldsymbol\theta)$, IDPR terms $\varphi^s(.,\gamma_{ij}=0 | .; \boldsymbol\theta)$ and IDOD terms $\varphi^d(\gamma_{ij}=1 | .; \boldsymbol\theta)$ in terms of $p(g, \mathbf{m}_{g \mathcal{N}(g)} | \BI(\BL_i); \boldsymbol\theta)$ by marginalization:
\begin{align}
\phi(i | \BI(\BL_i);\boldsymbol\theta) &= \log( p(g = i | \BI(\BL_i); \boldsymbol\theta) ) \\
\varphi^s(t_{ij},\gamma_{ij}=0 | \BI(\BL_i); \boldsymbol\theta) &= \log( p(m_{ij} = t_{ij} | g = i, \BI(\BL_i); \boldsymbol\theta) ) \\
\varphi^d(\gamma_{ij} =1 | \BI(\BL_i); \boldsymbol\theta) &= \log( p(m_{ij} = 0 | g = i, \BI(\BL_i); \boldsymbol\theta) )
\label{eq:marginalization}
\end{align}

\section{Inference}
\label{sec:inference}
To estimate the optimal configuration for each person, we search for the flexible composition $c \in \mathcal{C}_{\mathcal{G}}$ with the configurations of the locations $\BL$ and types $\BT$ that maximize the model score: $(c^*, {\BL^*}, {\BT^*}) = \arg \max_{c, \BL_, \BT} F(\BL, \BT, \mathcal{G}_c | \BI, \mathcal{G}) $.

Let $\mathcal{C}_{\mathcal{G}}^i \subset \mathcal{C}_{\mathcal{G}}$ be the subset of the flexible compositions that have node $i$ present (Obviously, $\cup_{i \in \mathcal{V}}\mathcal{C}_{\mathcal{G}}^i = \mathcal{C}_{\mathcal{G}}$), and we will consider the compositions that have the part with index $1$ present first, \ie, $\mathcal{C}_{\mathcal{G}}^1$. 

For all the flexible compositions $c \in \mathcal{C}_{\mathcal{G}}^1$, we set part $1$ as root. We will use dynamic programming to compute the best score over all these flexible compositions for each root location $\BL_1$. 

After setting the root, let $\mathcal{K}(i)$ be the set of children of part $i$ in the full graph ($\mathcal{K}(i) = \emptyset$, if part $i$ is a leaf). We use the following algorithm to compute the maximum score of all the flexible compositions $c \in \mathcal{C}_{\mathcal{G}}^1$: 
\begin{align}
S_i(\BL_i | \BI) &= A(\BL_i | \BI) + \sum_{k\in\mathcal{K}(i)} m_{ki}(\BL_i | \BI)
\label{eq:dp_score}
\\
B_{ij} &= b_j + \sum_{k\in\mathcal{K}(j)}B_{jk}
\\
m_{ki}(\BL_i | \BI) &= \max_{\gamma_{ik}}( (1-\gamma_{ik}) \times m_{ki}^{s}(\BL_i | \BI) \nonumber \\
&+ \gamma_{ik} \times m_{ki}^{d}(\BL_i | \BI) )
\label{eq:dp_message}
\\
m_{ki}^{s}(\BL_i | \BI) &= \max_{\BL_k, t_{ik}, t_{ki}} R(\BL_i, \BL_k, t_{ik}, t_{ki} | \BI) + S_k(\BL_k | \BI)  
\label{eq:dp_message_s}
\\
m_{ki}^{d}(\BL_i | \BI) &= D_{ik}(\gamma_{ik} = 1, \BL_i  | \BI) + B_{ik}
\label{eq:dp_message_d},
\end{align}
\noindent where $S_i(\BL_i | \BI)$ is the score of the subtree $\mathcal{T}_i$ with part $i$ each location $\BL_i$, and is computed by collecting the messages from all its children $k\in\mathcal{K}(i)$. Each child computes two kinds of messages $m_{ki}^s(\BL_i | \BI)$ and $m_{ki}^{d}(\BL_i | \BI)$ that convey information to parent for deciding whether to decouple it (and its followed subtree), \ie, Equation~\ref{eq:dp_message}. 

Intuitively, the message computed by Equation~\ref{eq:dp_message_s} measures how well we can find a child part $k$ that not only shows strong evidence of part $k$ (\eg, an elbow) and couples well with the other parts in the subtree $\mathcal{T}_k$ (\ie, $S_k(\BL_k | \BI)$), but also is suitable for the part $i$ at location $\BL_i$ based on the local image measurement (encoded in the IDPR terms). The message computed by Equation~\ref{eq:dp_message_d} measures the evidence to decouple $\mathcal{T}_k$ by combining the local image measurements around part $i$ (encoded in IDOD terms) and the learned occlusion decoupling bias. 

The following lemma states each $S_i(\BL_i | \BI)$ computes the maximum score for the set of flexible compositions $\mathcal{C}_{\mathcal{T}_i}^i$ that is within the subtree $\mathcal{T}_i$ and have part $i$ at $\BL_i$. In other words, we consider an object that is only composed with the parts in the subtree $\mathcal{T}_i$ (\ie, $\mathcal{T}_i$ is the full graph) and $\mathcal{C}_{\mathcal{T}_i}^i$ is the set of flexible compositions of the graph $\mathcal{T}_i$ that have part $i$ present. Since at root part (\ie, i = 1), we have $\mathcal{T}_1 = \mathcal{G}$, once the messages are passed to the root part,  $S_1(\BL_1 | \BI)$ gives the best score for all the flexible compositions in the full graph $c \in \mathcal{C}_{\mathcal{G}}^1$ that have part $1$ at $\BL_1$. 

\begin{lemma}
\begin{equation}
S_i(\BL_i, \BI) = \max_{ c \in \mathcal{C}_{\mathcal{T}_i}^i } \left\{ \max_{\BL_{/i}, \BT} F (\BL_i, \BL_{/i}, \BT,   \mathcal{G}_{c} | \BI, \mathcal{T}_i) \right\}  
\end{equation}
 \label{lemma:lm1}
\end{lemma}

\begin{proof}
We will prove the lemma by induction from leaf to root.


\noindent\textbf{Basis: } The proof is trivial when node $i$ is a leaf node.

\noindent\textbf{Inductive step: } Assume for each child $k \in \mathcal{K}(i)$ of the node $i$, the lemma holds. Since we do not consider the case that people are separated into isolated pieces, each flexible composition at node $i$ (\ie, $c \in\mathcal{C}_{\mathcal{T}_i}^i$) is composed of part $i$ and the flexible compositions from the children (\ie, $\mathcal{C}_{\mathcal{T}_k}^k, k \in \mathcal{K}(i)$) that are not decoupled. Since the graph is a tree, the best scores of the flexible compositions from each child can be computed separately, by $S_i(\BL_k, \BI), k \in \mathcal{K}(i)$ as assumed. These scores are then passed to node $i$ (Equation~\ref{eq:dp_message_s}). At node $i$ the algorithm can choose to decouple a child for better score (Equation~\ref{eq:dp_message}). Therefore, the best score at node $i$ is also computed by the algorithm. By induction, the lemma holds for all the nodes.
\end{proof}

By Lemma~\ref{lemma:lm1}, we can efficiently compute the best score for all the compositions with part $1$ present, \ie, $c \in\mathcal{C}_{\mathcal{G}}^1$, at each locations of part $1$ by dynamic programming (DP). These scores can be thresholded to generate multiple estimations with part $1$ present in an image. The corresponding configurations of locations and types can be recovered by the standard backward pass of DP until occlusion decoupling, \ie $\gamma_{ik} = 1$ in Equation~\ref{eq:dp_message}. All the decoupled parts are inferred as occluded and thus do not have location or pairwise type configurations. 

Since $\cup_{i\in\mathcal{V}} \mathcal{C}_{\mathcal{G}}^i = \mathcal{C}_{\mathcal{G}}$, we can get the best score for all the flexible compositions of the full graph $\mathcal{G}$ by computing the best score for each subset  $\mathcal{C}_{\mathcal{G}}^i, i \in \mathcal{V}$. More formally:
\begin{equation}
\begin{split}
\max_{c \in \mathcal{C}_{\mathcal{G}}, \BL, \BT} F (\BL, \BT,   \mathcal{G}_{c} | \BI, \mathcal{G}) 
= \max_{i \in \mathcal{V}} (\max_{c \in \mathcal{C}_{\mathcal{G}}^i, \BL, \BT} F (\BL, \BT,   \mathcal{G}_{c} | \BI, \mathcal{G})  )
\end{split}
\end{equation}

This can be done by repeating the DP procedure $K$ times, letting each part take its turn as the root. However, it turns out the messages on each edge only need to be computed twice, one for each direction. This allows us to implement an efficient message passing algorithm, which is of twice (instead of $K$ times) the complexity of the standard one-pass DP, to get the best score for all the flexible compositions.

\noindent\textbf{Computation:}
As discussed above, the inference is of twice the complexity of the standard one-pass DP. Moreover, the max operation over the locations $\BL_k$ in Equation~\ref{eq:dp_message_s}, which is a quadratic function of $\BL_k$, can be accelerated by the generalized distance transforms~\cite{felzenszwalb2005pictorial}. The resulting approach is very efficient, takes $O(2 T^2 LK)$ time once the image dependent terms are computed, where $T$ is the number of spatial relation types, $L$ is the total number of locations, and $K$ is the total number of parts in the model. This analysis assumes that all the pairwise spatial relationships have the same number of types, \ie, $T_{ij} = T_{ji} = T, \forall (i,j) \in \mathcal{E}$.

The computation of the image dependent terms is also efficient. They are computed over all the locations by a single DCNN. The DCNN is applied in a sliding window fashion by considering the fully-connected layers as $1\times 1$ convolutions~\cite{overfeat}, which naturally shares the computations common to overlapping regions.

\section{Learning}
\label{sec:learn}
We learn our model parameters from the images containing occluded people. The visibility of each part (or joint) is labeled, and the locations of the visible parts are annotated. We adopt a supervised approach to learn the model by first deriving the occlusion decoupling labels $\gamma_{ij}$ and type labels $t_{ij}$ from the annotations.

Our model consists of three sets of parameters: the mean relative positions $\mathbf{r} = \left\{\mathbf{r}_{ij}^{t_{ij}}, \mathbf{r}_{ji}^{t_{ji}} | (i,j) \in \mathcal{E} \right\}$ of different pairwise spatial relation types; the parameters $\boldsymbol\theta$ for the image dependent terms, \ie, the appearance terms, IDPR and IDOD terms; and the weight parameters $\mathbf{w}$(\ie, $w_i, \mathbf{w}_{ij}^{t_{ij}},w_{ij}, \mathbf{w}_{ji}^{t_{ji}},w_{ji}$), and bias parameters $\mathbf{b}$ (\ie, $b_k$). They are learnt separately by the K-means algorithm for $\mathbf{r}$, DCNN for $\boldsymbol\theta$, and linear Support Vector Machine (SVM)~\cite{cortes1995support} for $\mathbf{w}$ and $\mathbf{b}$.

\noindent\textbf{Derive Labels and Learn Mean Relative Positions:}
The ground-truth annotations give the part visibility labels $\mathbf{v}^n$, and locations $\BL^n$ for visible parts of each person instance $n \in \left\{1, \dots, N\right\}$. For each neighboring parts $(i,j) \in \mathcal{E}$, we derive $\gamma_{ij}^n=1$ if and only if part $i$ is visible but part $j$ is not, \ie, $v_i^n = 1$ and $v_j^n = 0$. Let $\mathbf{d}_{ij}$ be the relative position from part $i$ to its neighbor $j$, if both of them are visible. We cluster the relative positions over the training set $\left\{ \mathbf{d}_{ij}^n\ | v_i^n = 1, v_j^n = 1 \right\}$ to get $T_{ij}$ clusters (in the experiments $T_{ij}=8$ for all pairwise relations). Each cluster corresponds to a set of instances of part $i$ that share similar spatial relationship with its visible neighboring part $j$. Therefore, we define each cluster as a pairwise spatial relation type $t_{ij}$ from part $i$ to $j$ in our model, and the type label $t_{ij}^n$ for each training instance is derived based on its cluster index. The mean relative position $\mathbf{r}_{ij}^{t_{ij}}$ associated with each type is defined as the the center of each cluster. In our experiments, we use K-means by setting $\mbox{K}=T_{ij}$ to do the clustering.

\noindent\textbf{Parameters of Image Dependent Terms:}
After deriving the occlusion decoupling label and pairwise spatial type labels, each local image patch $\BI(\BL^n)$ centered at an annotated (visible) part location is labeled with category label $g^n \in \left\{1, \dots, K\right\}$, that indicates  which part is present, and also the type labels $\mathbf{m}_{g ^n\mathcal{N}(g^n)}^n$ that indicate its pairwise occlusion decoupling and spatial relationships with all its neighbors. In this way, we get a set of labeled patches $\left\{ \BI(\BL^n), g^n, {\mathbf{m}_{g^n \mathcal{N}(g^n)}^n} | v_{g^n}^n = 1\right\}$ from the visible parts of each labeled people, and also a set of background patches $\left\{ \BI(\BL^n), 0, 0 \right\} $ sampled from negative images, which do not contain people.

Given the labeled part patches and background patches, we train a $|\mathcal{U} |$-way DCNN classifier by standard stochastic gradient descent using softmax loss. The final $|\mathcal{U} |$-way softmax output is defined as our conditional probability distribution, \ie, $p(g, \mathbf{m}_{g \mathcal{N}(g)} | \BI(\BL_i); \boldsymbol\theta)$. See Section~\ref{sec:implement} for the details of our network.

\noindent\textbf{Weight and Bias Parameters:}
Given the derived occlusion decoupling labels $\gamma_{ij}$, we can associate each labeled pose with a flexible composition $c^n$. For the poses that is separated into several isolated compositions, we use the composition with the most number of parts. The location of each visible part in the associated composition $c^n$ is given by the ground-truth annotation, and the pairwise spatial types of it are derived above. We can then compute the model score of each labeled pose as a linear function of the parameters $\boldsymbol \beta = \left[\mathbf{w}, \mathbf{b}\right]$, so we use a linear SVM to learn these parameters:
\begin{equation*}
\begin{aligned}
& \min_{\boldsymbol\beta, \xi} 
& & \frac{1}{2}  \langle \boldsymbol \beta, \boldsymbol \beta \rangle + C \sum_{n} \xi_n \\
& \text{s.t.} 
& & \langle \boldsymbol \beta, \boldsymbol\Phi (c^n, \BI^n, \BL^n, \BT^n) \rangle + b_0 \geq 1 - \xi_n, \forall n \in \text{pos} \\
& & & \langle \boldsymbol \beta, \boldsymbol\Phi (c^n, \BI^n, \BL^n, \BT^n) \rangle + b_0 \leq -1 + \xi_n, \forall n \in \text{neg}
\end{aligned},
\label{eq:ssvm}
\end{equation*}

\noindent where $b_0$ is the scalar SVM bias, $C$ is the cost parameter, and $\boldsymbol\Phi(c^n, \BI^n, \BL^n, \BT^n)$ is a sparse feature vector representing the $n$-th example and is the concatenation of the image dependent terms (calculated from the learnt DCNN), spatial deformation features, and constants $1$s for the bias terms. The above constraints encourage the positive examples (pos) to be scored higher than 1 (the margin) and the negative examples (neg), which we mine from the negative images using the inference method described above, lower than -1. The objective function penalizes violations using slack variables $\xi_i$.

\section{Experiments}
\label{sec:experiments}
This section describes our experimental setup, presents comparison benchmark results, and gives diagnostic experiments.
\subsection{Dataset and Evaluation Metrics}
We perform experiments on the standard benchmarked dataset: ``We Are Family" Stickmen (WAF)~\cite{eichner2010we}, which contains challenging group photos, where several people often occlude one another (see Figure~\ref{fig:last}). The dataset contains 525 images with 6 people each on average, and is officially split into 350 images for training and 175 images for testing. Following \cite{chen_nips14, yang2013Pami}, we use the negative training images from the INRIAPerson dataset~\cite{dalal2005histograms} (These images do not contain people).  

We evaluate our method using the official toolkit of the dataset~\cite{eichner2010we} to allow comparison with previous work. The toolkit implements a version of occlusion-aware Percentage of Correct Parts (PCP) metric, where an estimated part is considered correctly localized if the \emph{average} distance between its endpoints (joints) and ground-truth is less than 50\% of the length of the ground-truth annotated endpoints, and an occluded body part is considered correct if and only if the part is also annotated as occluded in the ground-truth. 

We also evaluate the Accuracy of Occlusion Prediction (AOP) by considering occlusion prediction over all people parts as a binary classification problem. AOP does not care how well a part is localized, but is aimed to show the percentage of parts that have its visibility status correctly estimated. 

\subsection{Implementation detail}
\label{sec:implement}
\begin{figure}
\centering
\includegraphics[width=\linewidth,trim=5 5 5 5,clip=true]{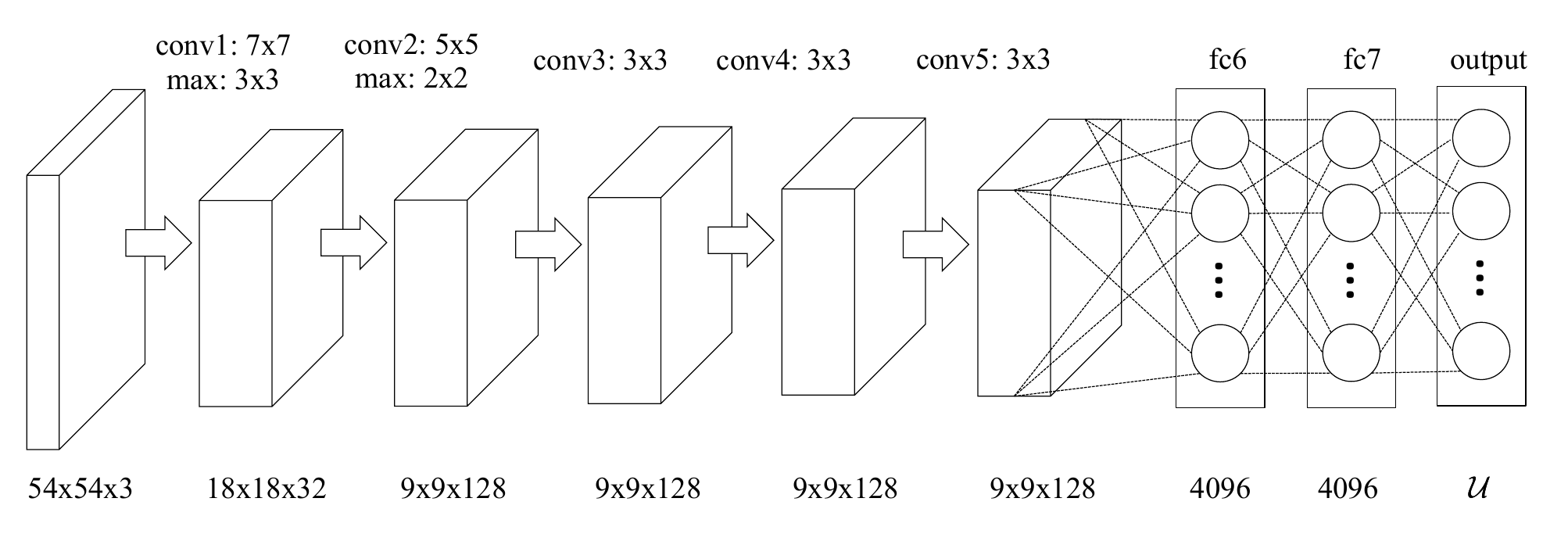}
\caption{An illustration of the DCNN architecture used in our experiments. It consists of five convolutional layers (conv), 2 max-pooling layers (max) and three fully-connected layers (fc) with a final $|\mathcal{U} |$-way softmax output. We use the rectification (ReLU) non-linearity, and the dropout technique described in \cite{krizhevsky2012imagenet}. }
\label{fig:dcnn_model}
\end{figure}

\noindent\textbf{DCNN Architecture:}
The layer configuration of our network is summarized in Figure~\ref{fig:dcnn_model}. In our experiments, the patch size of each part is $54 \times 54$. We pre-process each image patch pixel by subtracting the mean pixel value over all the pixels of training patches. We use the Caffe~\cite{jia2014caffe} implementation of DCNN. 

\noindent\textbf{Data Augmentation:}
We augment the training data by rotating and horizontally flipping the positive training examples to increase the number of training part patches with different spatial configurations with its neighbors. We follow~\cite{chen_nips14, yang2013Pami} to increase the number of parts by adding the midway points between annotated parts, which results in 15 parts on the WAF dataset. Increasing the number of parts produce more training patches for DCNN, which helps to reduce overfitting. Also covering a person with more parts is good for modeling foreshortening~\cite{yang2013Pami}.

\noindent\textbf{Part-based Non-Maximum Suppression:}
Using the proposed inference algorithm, a single image evidence of a part can be used multiple times in different estimations. This may produce duplicated estimations for the same person. We use a greedy part-based non-maximum suppression~\cite{chen_cvpr14} to prevent this. There is a score associated to each estimation. We sort the estimations by their score and start from the highest scoring estimation and remove the ones whose parts overlap significantly with the corresponding parts of any previously selected estimations. In the experiments, we require the interaction over union between the corresponding parts in different estimation to be less than 60\%.

\noindent\textbf{Other Settings:}
We use the same number of spatial relationship types for all pairs of neighbors in our experiments. They are set as $8$, \ie, $T_{ij} = T_{ji} = 8, \forall (i,j) \in \mathcal{E}$. 

\subsection{Benchmark results}
Table~\ref{tab:waf_compare} compares the performance of our method with the state of the art methods using the PCP and AOP metrics on the WAF benchmark, which shows our method improves the PCP performance on all parts, and significantly outperform the best previously published result~\cite{eichner2010we} by 11.3\% on mean PCP, and 4.9\% on AOP. Figure~\ref{fig:last} shows some estimation results of our model on the WAF dataset. 
\begin{table}
\scriptsize
\centering
\rowcolors{2}{}{gray!35}
\addtolength{\tabcolsep}{0.0pt}
\begin{tabular}{ c | c |c c c c  c }
\toprule[0.2 em] %
Method & AOP & Torso & Head & U.arms & L.arms & mPCP \\
Ours & \textbf{84.9} & \textbf{88.5} & \textbf{98.5} & \textbf{77.2} & \textbf{71.3} & \textbf{80.7} \\
\midrule
Multi-Person~\cite{eichner2010we} &  80.0 & 86.1 & 97.6 & 68.2 & 48.1 & 69.4\\
Ghiasi et. al.~\cite{Ghiasi14_people} & 74.0 & - & - & - & - & 63.6 \\
One-Person~\cite{eichner2010we} & 73.9 & 83.2 & 97.6 & 56.7 & 28.6 & 58.6 \\
\bottomrule[0.1 em]
\end{tabular}
\caption{Comparison of PCP and AOP on the WAF dataset. Our method improves the PCP performance on all parts, and significantly outperform the best previously published result~\cite{eichner2010we} by 11.3\% on mean PCP, and 4.9\% on AOP. }
\label{tab:waf_compare}
\end{table}

\begin{figure*}
\centering
\includegraphics[width=\linewidth,trim= 5 5 5 5,clip=true]{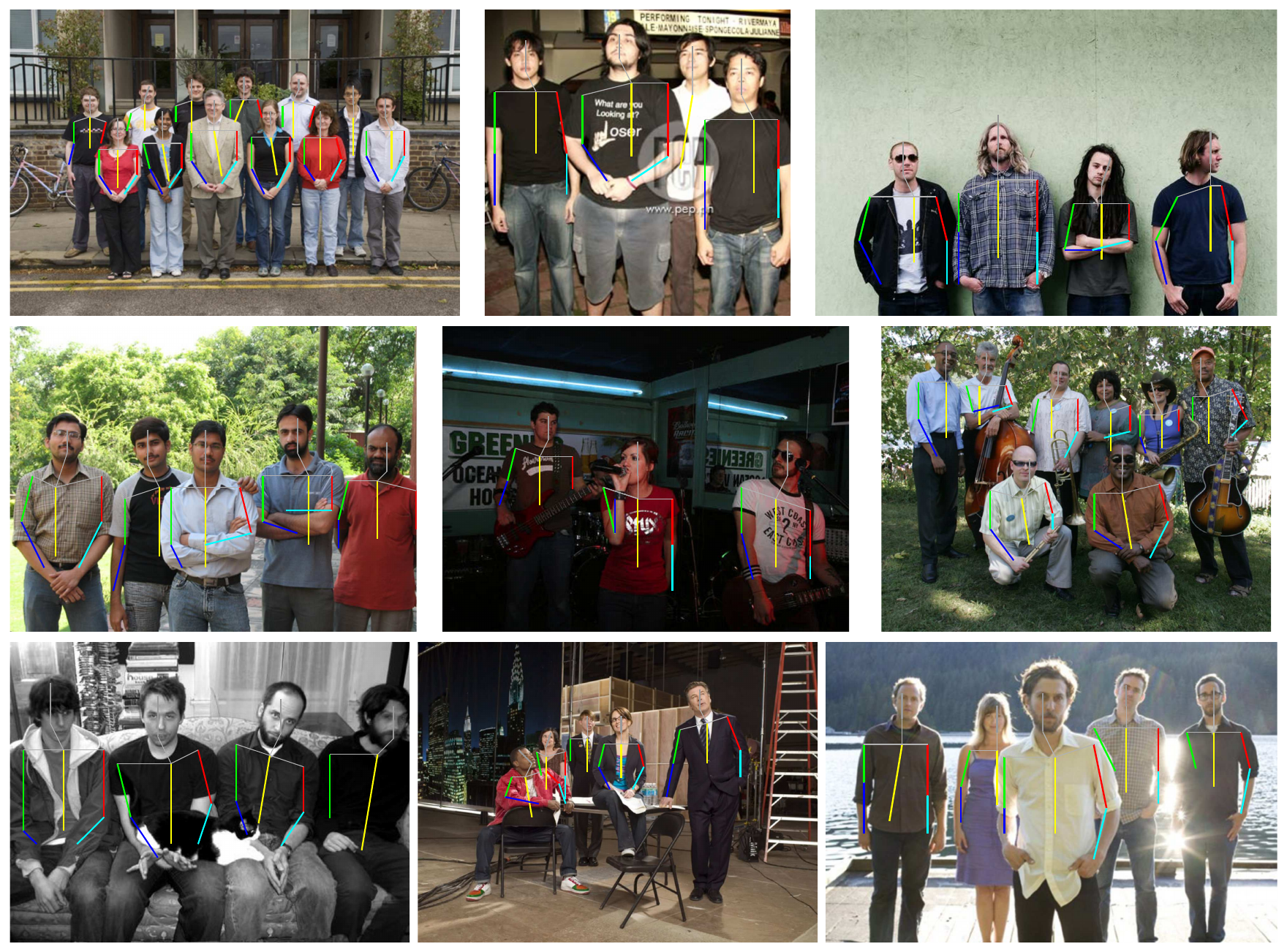}
\caption{Results on the WAF dataset. We show the parts that are inferred as visible, and thus have estimated configurations, by our model.}
\label{fig:last}
\end{figure*}

\subsection{Diagnostic Experiments}
\label{sec:diag}
\noindent\textbf{Connectivity Prior Verification:}
We analyze the test set of the WAF dataset using ground truth annotation, and find that 95.1\% of the people instances have their visible parts form a connected subtree. This verifies our connectivity prior that visible parts of a human tend to form a connected subtree, even in the presence of significant occlusion.

\noindent\textbf{Term Analysis:}
We design the following experiments to better understand each design component in our model. 

Firstly, our model is designed to handle different degrees of occlusion by efficiently searching over large number of flexible compositions. When we do not consider occlusion and use a single composition (\ie, the full graph), our model reduces to the base model~\cite{chen_nips14}. So we perform an diagnostic experiment by using the base model~\cite{chen_nips14} on the WAF dataset, which will infer every part as visible and localize them. 

Secondly, we model occlusion by combining the cue from absence of evidence for body part and local image measurement around the occlusion boundary, which is encoded in the IDOD terms. So we perform the second diagnostic experiment by removing the IDOD terms(\ie, in Equation~\ref{eq:idod}, we have $\varphi^d (\gamma_{ij}=1|. ; \boldsymbol\theta) = 0$). In this case, the model handles occlusion only by exploiting the cue from absence of evidence for body part.

We show the PCP and AOP performance of the diagnostic experiments in Table~\ref{tab:waf_diag}. As is shown, flexible compositions (\ie, \textit{FC}) significantly outperform a single composition (\ie, the base model~\cite{chen_nips14}), and adding \textit{IDOD} terms improves the performance significantly (see the caption for details).

\begin{table}
\scriptsize
\centering
\rowcolors{2}{}{gray!35}
\addtolength{\tabcolsep}{-0.1pt}
\begin{tabular}{ c | c | c c c c  c}
\toprule[0.2 em] %
Method & AOP & Torso & Head & U.arms & L.arms & mPCP \\
Base Model~\cite{chen_nips14} & 73.9 & 81.4 & 92.6 & 63.6 & 47.6 & 66.1  \\
\textit{FC} & 82.0 & 87.0 & \textbf{98.6} & 72.7 & 67.5 & 77.7 \\ 
\textit{FC}+\textit{IDOD} & \textbf{84.9} & \textbf{88.5} & 98.5 & \textbf{77.2} & \textbf{71.3} & \textbf{80.7} \\
\bottomrule[0.1 em]
\end{tabular}
\caption{Diagnostic Experiments PCP and AOP results on the WAF dataset. Using flexible compositions (\ie, \textit{FC}) significantly improves our base model~\cite{chen_nips14} by 11.6\% on PCP and 8.1\% on AOP. Adding \textit{IDOD} terms (\textit{FC}+\textit{IDODs}, \ie, the full model) further improves our PCP performance to 80.7\% and AOP performance to 84.9\%, which is significantly higher than the state of the art methods.}
\label{tab:waf_diag}
\end{table}

\section{Conclusion}
This paper develops a new graphical model for parsing people. We introduce and experimentally verify on the WAF dataset (see Section~\ref{sec:diag}) a novel prior that the visible body parts of human tend to form a connected subtree, which we define as a flexible composition, even with the presence of significant occlusion. This is equivalent to modeling people as a mixture of flexible compositions. We define novel occlusion cues and learn them from data. We show very efficient inference can be done for our model by exploiting part sharing so that computing over all the flexible compositions takes only twice that of the base model~\cite{chen_nips14}.  We evaluate on the WAF dataset and show we significantly outperform current state of the art methods~\cite{eichner2010we, Ghiasi14_people}. We also show big improvement over our base model, which does not model occlusion explicitly. 

\section{Acknowledgements}
This research has been supported by the Center for Minds, Brains and Machines (CBMM), funded by NSF STC award CCF-1231216, and the grant ONR N00014-12-1-0883. The GPUs used in this research were generously donated by the NVIDIA Corporation.

{\small
\bibliographystyle{ieee}
\bibliography{biblio}

\begin{thebibliography}{10}\itemsep=-1pt

\bibitem{azizpour2012object}
H.~Azizpour and I.~Laptev.
\newblock Object detection using strongly-supervised deformable part models.
\newblock In {\em European Conference on Computer Vision (ECCV)}, 2012.

\bibitem{BourdevMalikICCV09}
L.~Bourdev and J.~Malik.
\newblock Poselets: Body part detectors trained using 3d human pose
  annotations.
\newblock In {\em International Conference on Computer Vision (ICCV)}, 2009.

\bibitem{burgos2013robust}
X.~P. Burgos-Artizzu, P.~Perona, and P.~Doll{\'a}r.
\newblock Robust face landmark estimation under occlusion.
\newblock In {\em International Conference on Computer Vision (ICCV)}, 2013.

\bibitem{Jay_George15_ICLR}
L.-C. Chen, G.~Papandreou, I.~Kokkinos, K.~Murphy, and A.~Yuille.
\newblock Semantic image segmentation with deep convolutional nets and fully
  connected crfs.
\newblock In {\em International Conference on Learning Representations (ICLR)},
  2015.

\bibitem{chen_cvpr14}
X.~Chen, R.~Mottaghi, X.~Liu, S.~Fidler, R.~Urtasun, and A.~Yuille.
\newblock Detect what you can: Detecting and representing objects using
  holistic models and body parts.
\newblock In {\em Computer Vision and Pattern Recognition (CVPR)}, 2014.

\bibitem{chen_nips14}
X.~Chen and A.~Yuille.
\newblock Articulated pose estimation by a graphical model with image dependent
  pairwise relations.
\newblock In {\em Advances in Neural Information Processing Systems (NIPS)},
  2014.

\bibitem{cortes1995support}
C.~Cortes and V.~Vapnik.
\newblock Support-vector networks.
\newblock {\em Machine learning}, 1995.

\bibitem{coughlan2000efficient}
J.~Coughlan, A.~Yuille, C.~English, and D.~Snow.
\newblock Efficient deformable template detection and localization without user
  initialization.
\newblock {\em Computer Vision and Image Understanding}, 2000.

\bibitem{dalal2005histograms}
N.~Dalal and B.~Triggs.
\newblock Histograms of oriented gradients for human detection.
\newblock In {\em Computer Vision and Pattern Recognition (CVPR)}, 2005.

\bibitem{Dollar2012PAMI}
P.~Doll\'ar, C.~Wojek, B.~Schiele, and P.~Perona.
\newblock Pedestrian detection: An evaluation of the state of the art.
\newblock {\em PAMI}, 34, 2012.

\bibitem{eichner2010we}
M.~Eichner and V.~Ferrari.
\newblock We are family: Joint pose estimation of multiple persons.
\newblock In {\em European Conference on Computer Vision (ECCV)}, 2010.

\bibitem{felzenszwalb2010object}
P.~F. Felzenszwalb, R.~B. Girshick, D.~McAllester, and D.~Ramanan.
\newblock Object detection with discriminatively trained part-based models.
\newblock {\em IEEE Transactions on Pattern Analysis and Machine Intelligence
  (TPAMI)}, 2010.

\bibitem{felzenszwalb2005pictorial}
P.~F. Felzenszwalb and D.~P. Huttenlocher.
\newblock Pictorial structures for object recognition.
\newblock {\em International Journal of Computer Vision (IJCV)}, 2005.

\bibitem{ferrari2008progressive}
V.~Ferrari, M.~Marin-Jimenez, and A.~Zisserman.
\newblock Progressive search space reduction for human pose estimation.
\newblock In {\em Computer Vision and Pattern Recognition (CVPR)}, 2008.

\bibitem{fischler1973representation}
M.~A. Fischler and R.~A. Elschlager.
\newblock The representation and matching of pictorial structures.
\newblock {\em IEEE Transactions on Computers}, 1973.

\bibitem{Ghiasi14_face}
G.~Ghiasi and C.~C. Fowlkes.
\newblock Occlusion coherence: Localizing occluded faces with a hierarchical
  deformable part model.
\newblock In {\em Computer Vision and Pattern Recognition (CVPR)}, 2014.

\bibitem{Ghiasi14_people}
G.~Ghiasi, Y.~Yang, D.~Ramanan, and C.~C. Fowlkes.
\newblock Parsing occluded people.
\newblock In {\em Computer Vision and Pattern Recognition (CVPR)}, 2014.

\bibitem{girshick2014rich}
R.~Girshick, J.~Donahue, T.~Darrell, and J.~Malik.
\newblock Rich feature hierarchies for accurate object detection and semantic
  segmentation.
\newblock In {\em Computer Vision and Pattern Recognition (CVPR)}, 2014.

\bibitem{girshick2011object}
R.~B. Girshick, P.~F. Felzenszwalb, and D.~A. Mcallester.
\newblock Object detection with grammar models.
\newblock In {\em Advances in Neural Information Processing Systems (NIPS)},
  2011.

\bibitem{hsiao2012occlusion}
E.~Hsiao and M.~Hebert.
\newblock Occlusion reasoning for object detection under arbitrary viewpoint.
\newblock In {\em Proceedings of IEEE Conference on Computer Vision and Pattern
  Recognition (CVPR)}. IEEE, 2012.

\bibitem{jia2014caffe}
Y.~Jia, E.~Shelhamer, J.~Donahue, S.~Karayev, J.~Long, R.~Girshick,
  S.~Guadarrama, and T.~Darrell.
\newblock Caffe: Convolutional architecture for fast feature embedding.
\newblock {\em arXiv preprint arXiv:1408.5093}, 2014.

\bibitem{Johnson10}
S.~Johnson and M.~Everingham.
\newblock Clustered pose and nonlinear appearance models for human pose
  estimation.
\newblock In {\em British Machine Vision Conference (BMVC)}, 2010.

\bibitem{kanizsa1979organization}
G.~Kanizsa.
\newblock {\em Organization in vision: Essays on Gestalt perception}.
\newblock Praeger New York, 1979.

\bibitem{krizhevsky2012imagenet}
A.~Krizhevsky, I.~Sutskever, and G.~E. Hinton.
\newblock Imagenet classification with deep convolutional neural networks.
\newblock In {\em Advances in Neural Information Processing Systems (NIPS)},
  2012.

\bibitem{pishchulin2013poselet}
L.~Pishchulin, M.~Andriluka, P.~Gehler, and B.~Schiele.
\newblock Poselet conditioned pictorial structures.
\newblock In {\em Computer Vision and Pattern Recognition (CVPR)}, 2013.

\bibitem{sapp2010adaptive}
B.~Sapp, C.~Jordan, and B.~Taskar.
\newblock Adaptive pose priors for pictorial structures.
\newblock In {\em Computer Vision and Pattern Recognition (CVPR)}, 2010.

\bibitem{sapp2013modec}
B.~Sapp and B.~Taskar.
\newblock Modec: Multimodal decomposable models for human pose estimation.
\newblock In {\em Computer Vision and Pattern Recognition (CVPR)}, 2013.

\bibitem{overfeat}
P.~Sermanet, D.~Eigen, X.~Zhang, M.~Mathieu, R.~Fergus, and Y.~LeCun.
\newblock Overfeat: Integrated recognition, localization and detection using
  convolutional networks.
\newblock In {\em International Conference on Learning Representations (ICLR)},
  2014.

\bibitem{tompson2014joint}
J.~J. Tompson, A.~Jain, Y.~LeCun, and C.~Bregler.
\newblock Joint training of a convolutional network and a graphical model for
  human pose estimation.
\newblock In {\em Advances in Neural Information Processing Systems (NIPS)},
  2014.

\bibitem{tu2006parsing}
Z.~Tu and S.-C. Zhu.
\newblock Parsing images into regions, curves, and curve groups.
\newblock {\em International Journal of Computer Vision}, 2006.

\bibitem{wang2013approach}
C.~Wang, Y.~Wang, and A.~L. Yuille.
\newblock An approach to pose-based action recognition.
\newblock In {\em Computer Vision and Pattern Recognition (CVPR)}, 2013.

\bibitem{yang2013Pami}
Y.~Yang and D.~Ramanan.
\newblock Articulated human detection with flexible mixtures of parts.
\newblock {\em IEEE Transactions on Pattern Analysis and Machine Intelligence
  (TPAMI)}, 2013.

\bibitem{yao2012coupled}
A.~Yao, J.~Gall, and L.~Van~Gool.
\newblock Coupled action recognition and pose estimation from multiple views.
\newblock {\em International journal of computer vision}, 2012.

\bibitem{zhang2014panda}
N.~Zhang, M.~Paluri, M.~Ranzato, T.~Darrell, and L.~Bourdev.
\newblock Panda: Pose aligned networks for deep attribute modeling.
\newblock In {\em Computer Vision and Pattern Recognition (CVPR)}, 2014.

\bibitem{zhu2008max}
L.~Zhu, Y.~Chen, Y.~Lu, C.~Lin, and A.~Yuille.
\newblock Max margin and/or graph learning for parsing the human body.
\newblock In {\em Computer Vision and Pattern Recognition (CVPR)}, 2008.

\bibitem{zhu2010part}
L.~Zhu, Y.~Chen, A.~Torralba, W.~Freeman, and A.~Yuille.
\newblock Part and appearance sharing: Recursive compositional models for
  multi-view.
\newblock In {\em Computer Vision and Pattern Recognition (CVPR)}, 2010.

\bibitem{zhu2006stochastic}
S.-C. Zhu and D.~Mumford.
\newblock A stochastic grammar of images.
\newblock {\em Foundations and Trends in Computer Graphics and Vision}, 2006.

\end{thebibliography}
}
\end{document}